\documentclass[12pt]{article}
\usepackage{verbatim}
\usepackage{setspace}
\usepackage{graphicx}
\usepackage[centertags]{amsmath}
\usepackage{amsfonts}
\usepackage{multirow}
\usepackage{amssymb}
\usepackage{amsmath}
\usepackage{amsthm}
\theoremstyle{plain}
\newtheorem{thm}{Theorem}[section]

\newtheorem{lem}[thm]{Lemma}

\theoremstyle{definition}
\newtheorem{defn}{Definition}[section]
\theoremstyle{remark}

\title{Least Absolute Gradient Selector: variable selection via Pseudo-Hard Thresholding}
\author{Kun Yang\footnote{Kun Yang (Email: \texttt{kunyang@stanford.edu}) is a PhD student at Institute for Computational and Mathematical Engineering, Stanford University. Kun Yang is partially supported by General Wang Yaowu Fellowship.}}
\begin{document}
\maketitle
\begin{abstract}
In this paper, we propose a new approach, called the \textbf{LAGS}, short for ``least absulute gradient selector'', to this challenging yet interesting problem by mimicking the discrete selection process of $l_0$ regularization in linear regression. To estimate $\beta$ under the influence of noise, we consider, nevertheless, the following convex program
\begin{equation*}
\hat{\beta} = \textrm{arg min}\frac{1}{n}\|X^{T}(y - X\beta)\|_1 + \lambda_n\sum_{i = 1}^pw_i(y;X;n)|\beta_i|
\end{equation*}
$\lambda_n > 0$ controls the sparsity and $w_i > 0$ dependent on $y, X$ and $n$ is the weights on different $\beta_i$; $n$ is the sample size. Surprisingly, we shall show in the paper, both geometrically and analytically,  that LAGS enjoys two attractive properties: (1) LAGS demonstrates discrete selection behavior and hard thresholding property as $l_0$ regularization by strategically chosen $w_i$, we call this property \emph{``pseudo-hard thresholding''}; (2) Asymptotically, LAGS is consistent and capable of discovering the true model; nonasymptotically, LAGS is capable of identifying the sparsity in the model and the prediction error of the coefficients is bounded at the noise level up to a logarithmic factor---$\log p$, where $p$ is the number of predictors.
\par Computationally, LAGS can be solved efficiently by convex program routines for its convexity or by simplex algorithm after recasting it into a linear program. The numeric simulation shows that LAGS is superior compared to soft-thresholding methods in terms of mean squared error and parsimony of the model.
\end{abstract}
\section{Introduction}\label{intro}
One of the most widely used model in statistics is the linear regression. In many applications, scientists are interested in estimating a mean response $X\beta$ from the data $y = (y_1, y_2, ..., y_n)$. The $p-$dimensional parameter of interest $\beta$ are estimated from the linear model
\begin{equation}
y = \alpha + X\beta + \epsilon
\label{linsetup}
\end{equation}
where $\alpha$ is the intercept, $\epsilon$ the noise. A common assumption is that $\epsilon$ is Gaussian with $\epsilon_i \sim \mathcal{N}(0, \sigma^2)$, but this is not an essential requirement as our methods are applicable to other types of noise which have heavier tails.
\par Scientists usually have no information of the underlying models, a large number of predictors are chosen in initial stage to attenuate the possible bias and variance as well as to enhance predictability. Regression procedures capable of identifying the explanatory variables (others are set to be or shrunken to $0$) and obtain good estimates of the response play a pivotal role. Ideally, Best Subset regression such AIC \cite{Akaike1974}, $C_p$ \cite{Mallows2000}, BIC \cite{Schwarz1978} and RIC \cite{Foster1994} achieve the trade-off between model complexity and goodness of fit. These estimators are essentially least square penalized by $l_0$ norm of $\beta$ with different control coefficients. A standard formulation of $l_0$ penalized least square is
\begin{equation}\label{hardth}
(\hat{\alpha}, \hat{\beta}) = \textrm{arg min}\frac{1}{2}\|y-\alpha\mathbf{1}-X\beta\|_2^2 + \lambda\|\beta\|_0
\end{equation}
$\lambda$ is the control coefficient. For the rest of the paper, we assume that columns of $X$ are standardized and $y$ is centered. Under this setting $\alpha = 0$, hence we omit it.
\par  $l_0$ penalized least square has the hard thresholding property that keeps the large coefficient intact while sets small ones to be zero. However, unfortunately, solving (\ref{hardth}) is a discrete process which needs to enumerate all the possible subset of $\beta$; the combinatorial nature of $l_0$ norm limits the application of (\ref{hardth}) when the number of predictors is large. As a compromise, convex relaxation to $l_1$ norm such as the Lasso \cite{Tibshirani1996} is a widely used technique for simultaneously estimation and variable selection. The Lasso is
\begin{equation}
\label{lasso}
\hat{\beta} = \textrm{arg min}\frac{1}{2}\|y - X\beta\|_2^2 + \lambda\|\beta\|_1
\end{equation}
The $l_1$ penalty shrinks $\beta$ towards zero, and also sets many coefficients to be exactly zero. Thus, the Lasso often regards as the substitute of (\ref{hardth}).
\par The Lasso and the subsequently appeared methods, such as the LARS \cite{Efron2004}, elastic-net \cite{Zou2005}, adaptive Lasso \cite{Zou2006}, Dantzig Selector \cite{Candes2007} to name a few, closely relate to the soft thresholding in signal processing \cite{Donoho1995} in contrast to hard thresholding as in  (\ref{hardth}). Since the soft thresholding both shrinks and selects, it often results in a model more complicated than the true model in its effort to spread the penalty among the predictors. As a greedier attempt, SCAD \cite{Fan2001,Fan2010,Fan2008} and SparseNet \cite{Mazumder2011}penalize the loss function by non-convex penalties. Similar to (\ref{lasso}), $\beta$ is estimated by
\begin{equation}\label{SCAD}
\hat{\beta} = \textrm{arg min}\frac{1}{2}\|y-X\beta\|_2^2 + \lambda\sum_{i = 1}^pP(|\beta_i|;\lambda;\gamma)
\end{equation}
where $P(|\beta_i|;\lambda;\gamma)$ defines a family of penalty functions concave in $\beta$, and $\lambda$ and $\gamma$ controls the sparsity and concavity. It has been shown that (\ref{SCAD}) enjoys better variable selection properties compared to $l_1$ relaxation; whereas, both algorithms cannot assure to find the global optimal.
\par In this paper, we propose a new approach, called the \textbf{LAGS}, short for ``least absolute gradient selector'', to this challenging yet interesting problem by mimicking the discrete selection process of $l_0$ regularization. To estimate $\beta$ under the influence of noise, we consider, nevertheless, the following convex program
\begin{equation}\label{lags}
\hat{\beta} = \textrm{arg min}\frac{1}{n}\|X^{T}(y - X\beta)\|_1 + \lambda_n\sum_{i = 1}^pw_i(y;X;n)|\beta_i|
\end{equation}
$\lambda_n > 0$ controls the sparsity and $w_i > 0$ dependent on $y, X$ and $n$ is the weights on different $\beta_i$; $n$ is the sample size. Surprisingly, we shall show in the following sections, both geometrically and analytically,  that (\ref{lags}) demonstrates discrete selection behavior as $l_0$ penalty and hard thresholding property by strategically chosen $w_i$, we call this property \emph{``pseudo-hard thresholding''}. The graphical comparison of hard thresholding and pseudo-hard thresholding is given in prostate cancer example in section 7.
\par The rest of the paper is organized as follows: section 2 presents the motivation of LAGS and connects the ideas with previous work; section 3 establishes the theorem regarding the properties of LAGS and highlights the ``pseudo-hard thresholding''; section 4 shows the potential problems associated with the Dantzig Selector and provides a neat way to choose $w_i$; section 5 are the proofs of the theorems; section 6 discusses the computational issue of how to solve LAGS; section 7 demonstrates it by numeric examples; discussion and future work are in section 8.
\section{The LAGS}
\subsection{Insight from orthonormal case}
The properties of hard thresholding can be better understood when the design $X$ is orthonormal, i.e., $X^TX = I$. The solution for (\ref{hardth}) is
\begin{equation}\label{hardsol}
\hat{\beta_j}^{l_0} = \beta_j^{o}\mathbf{I}(|\beta_j^o| \geq \lambda)
\end{equation}
as a reference, the Lasso solution is
\begin{equation}\label{lassosol}
\hat{\beta_j}^{lasso} = \textrm{sign}(\beta_j^{o})(|\beta_j^o| - \lambda)_+
\end{equation}
where $\beta^o = X^Ty$ is the ordinary least square (OLS) estimate. Notice that the hard thresholding operator (\ref{hardsol}) is discontinuous with the jump at $|\beta_j^o| = \lambda$, while the soft thresholding operator (\ref{lassosol}) is continuous. If we assume $|\beta_1^o| \geq ...|\beta_{p_0}^o| > |\beta_{p_0 + 1}^o| ...\geq |\beta_{p}^o|$, another property of (\ref{hardsol}) is that any $\lambda\in (|\beta_{p_0}^o|, |\beta_{p_0 + 1}^o|)$ will keep the first $p_0$ coefficients. Based on this property, we define ``pseudo-hard thresholding'' as
\begin{defn}\label{pseudo}
  A penalized regression has ``pseudo-hard thresholding'' property if there exist some intervals of $R^+$, such that changing sparsity parameter $\lambda$ in these intervals will keep the coefficients $\beta$ unaltered.
\end{defn}
In order to achieve pseudo-hard thresholding in convex world, consider function of $\theta$
\begin{equation}\label{lagseq}
f(\theta;z) = |z - \theta| + \frac{\gamma}{\lambda}|\theta|
\end{equation}
Minimizing $f(\theta;z)$, the solution is
\begin{equation}
\label{lagsol}
\hat{\theta} = \left\{
\begin{array}{ll}
z & \gamma \leq \lambda\\
0 & \gamma > \lambda
\end{array}
\right.
\end{equation}
\begin{figure}
  \center
  \caption{Plot of $f(\theta;z)$. The left panel has $\gamma / \lambda < 1$; the right panel $\gamma / \lambda >1$. It is obvious that the minimum is obtained at $x = 1$ on the left, while $x = 0$ on the right}
  \includegraphics[width = 4in, height = 2in]{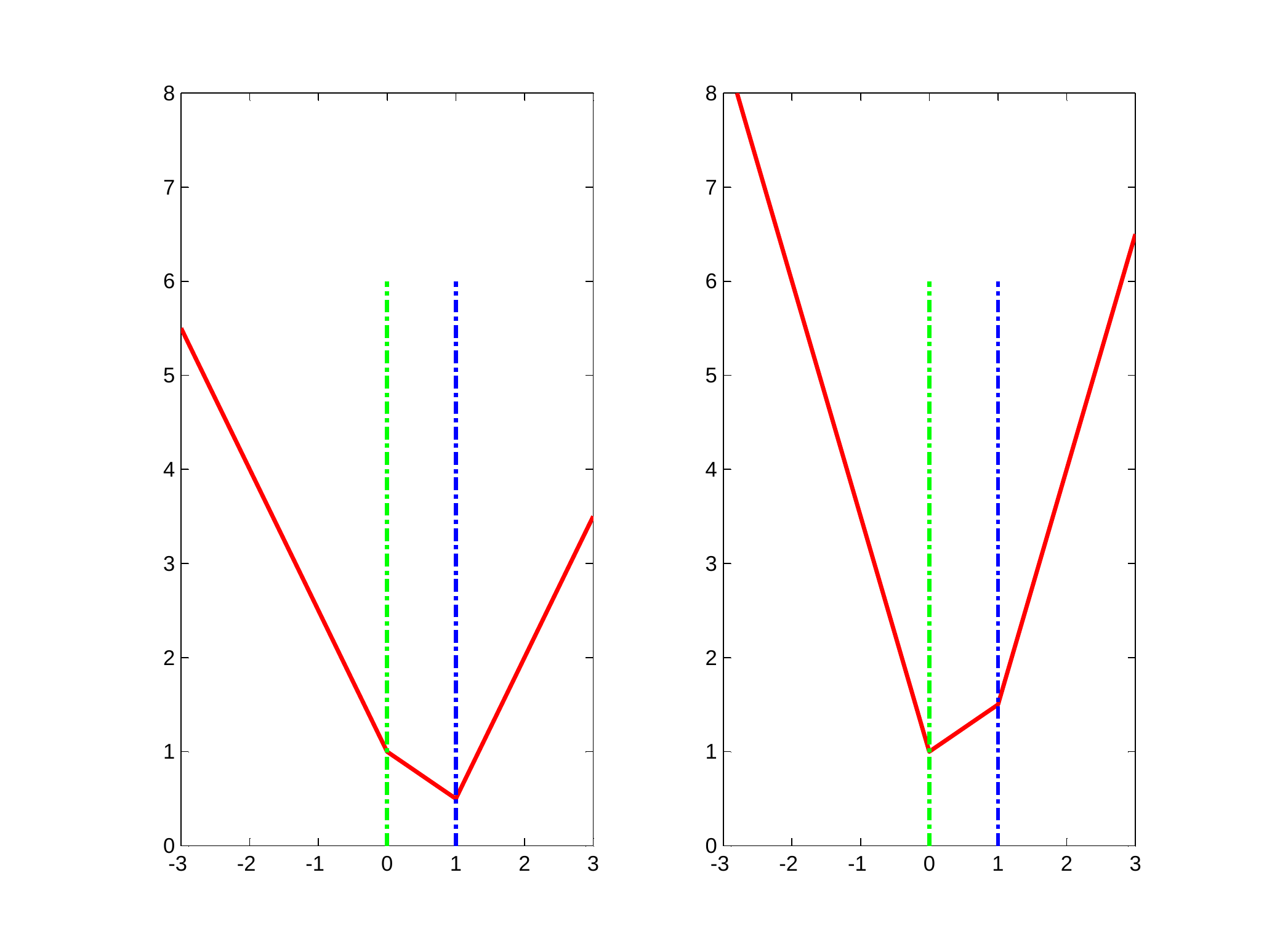}
  \label{lags_func}
\end{figure}
$\hat{\theta}$ is a discontinuous function of $\gamma$, with the breaking point at $\gamma = \lambda$. Motivated by this special property of (\ref{lagseq}), we formulate LAGS as in (\ref{lags}). As a matter of fact, the idea to minimize $\|X^T(y - X\beta)\|$ is pioneered by \cite{Candes2007} in Dantzig Selector
\begin{equation}
  \label{Dantzig1}
  \hat{\beta} = \textrm{arg }\min\|\beta\|_1 \quad\quad\textrm{subject to}\quad \|X^T(y - X\beta)\|_{\infty} \leq t
\end{equation}
which can be written equivalently as
\begin{equation}
  \label{Dantzig2}
  \hat{\beta} = \textrm{arg }\min\|X^T(y - X\beta)\|_{\infty} \quad\quad\textrm{subject to}\quad \|\beta\|_1 \leq t
\end{equation}
One reason why $\|X^T(y - X\beta)\|$ should be small given in \cite{Candes2007} is that a good estimate of $\beta$ should be independent of orthogonal transformations. Another more compelling yet insightful argument is that in OLS one needs to minimize $f_{OLS} = \frac{1}{2}\|y - X\beta\|_2^2$, its gradient is $\nabla f_{OLS} = -X^T(y - X\beta)$; solving $\nabla f_{OLS} = 0$ results in OLS estimates, hence, one can expect that a good $\hat{\beta}$ should shrink $\nabla f_{OLS}$ towards $\mathbf{0}$. In order to incorporate (\ref{lagseq}), $l_1$ is chosen here---from where the name ``LAGS'' comes; we will show in the following sections that this choice of norm can set some elements of the gradient to zero, which means unrestricted coefficients for them, like in $l_0$ penalty. Moreover, if we instead consider the absolute deviance by substituting $\|X^T(y - X\beta)\|_1$ with $\|y - X\beta\|_1$, it is the LAD-Lasso \cite{Wang2007}; but LAD-Lasso has no hard thresholding property even in orthonormal design case.
\subsection{The weight $w_i$ matters}
In the orthonormal case, (\ref{lags}) becomes
\begin{equation}
  \label{ortholags}
  \hat{\beta_i} = \textrm{arg min}|\beta_i^o - \hat{\beta}_i| + \lambda w_i|\hat{\beta}_i|, i = 1, ..., p
\end{equation}
One heuristic to choose $w_i$ is to consider the correlation $c_i$ between $y$ and $x_i$---the $i$th column of $X$: if $|c_i|$ is large, which means $i$th predictor may be a good explanatory variable, hence $\beta_i$ should be penalized less; otherwise, it should be penalized more. Thus, we set
\begin{equation}
  \label{W1CH}
  w_i = \frac{1}{|c_i|}
\end{equation}
with a little abuse of notation, $w_i = \infty$ when $c_i = 0$. Without loss of generality, we assume that $|c_1| \geq ...|c_{p_0}| > ...\geq|c_p|$, which implies $w_1 \leq ...  w_{p_0} < ...\leq w_p$. By choosing $\lambda$, s.t. $w_{p_0 + 1}^{-1} < \lambda < w_{p_0}^{-1}$, we have $\hat{\beta}_i = \beta_i^o, i = 1, ..., p_0$; $\hat{\beta}_i = 0, i = p_0 + 1, ..., p$, which is pseudo-hard thresholding and $\hat{\beta}$ is identical with (\ref{hardsol}).
\par To push this heuristic further, we notice that $c_i$ is the coefficient of OLS in the orthonormal design case. This suggests that we can choose $w_i^{-1}$ as absolute value of OLS coefficients $\beta_i^{o}$, see section 4 for detail.
\section{Properties of LAGS}
In section 2, some properties of LAGS are demonstrated in the simplest case. These properties are not incidental. To formally state our results, we decompose the regression coefficient as $\beta = (\beta^{(1)}, \beta^{(2)})$, where $\beta^{(1)} = (\beta_1, ..., \beta_{p_0})$ corresponds to the true parameters and $\beta^{(2)} = (\beta_{p_0 + 1}, ..., \beta{p})$ are redundant; the columns of $X$ are decomposed alike. Furthermore, define
\begin{equation}
  \label{an}
  a_n = \max_{1\leq j\leq p_0} \{w_i^n\}
\end{equation}
and
\begin{equation}
  \label{bn}
  b_n = \min_{p_0 + 1\leq j\leq p}\{w_i^n\}
\end{equation}
We assume three conditions:
\begin{itemize}
  \item[(a)] $y = x^{(1)}\beta^{(1)} + \epsilon$, where $\epsilon$ is noise with mean 0 and variance $\sigma^2$. \\
  \item[(b)] $C_n = \frac{1}{n}X^TX\rightarrow C$ and $C_n = \left[\begin{array}{cc}C_{11}^n & C_{12}^n\\C_{21}^n & C_{22}^n\end{array}\right]$, $C = \left[\begin{array}{cc}C_{11} & C_{12}\\C_{21} & C_{22}\end{array}\right]$, where $C_n$ and $C$ are positive definite matrices.\\
  \item[(c)] $\|C_{11}^{-1}C_{12}\|_{\infty} \leq 1-\eta$ and $\|(C_{11}^n)^{-1}C_{12}^n\|_{\infty} \leq 1-\eta_n$, where $\eta$ and $\eta_n$ are positive constants. This condition is established in \cite{Zou2006} and also called Irrespresentable Condition in \cite{Zhao2007}.
\end{itemize}
$C$ is actually the covariance matrix of predictors. Consider set $\Omega = \{s\in\mathbb{R}^p: \|s\|_{\infty} = 1, \max\{|s_1|, ..., |s_{p_0}|\} = 1\}$, which is closed and contains in the unit sphere under $l_{\infty}$ norm, hence $\Omega$ is a compact set. Let us define
\begin{equation}
\label{defgamma}
\gamma = \min_{s\in\Omega}\{\|[C_{11}, C_{12}](s^{(1)}, s^{(2)})^T\|_{\infty}\}
\end{equation}
here we partition $s$ as above. It is obvious that $\|s^{(1)}\|_{\infty} = 1$ and $\|s^{(2)}\|_{\infty} \leq 1$, thus $\|[I, C_{11}^{-1}C_{12}]s\|_{\infty}\geq\|s^{(1)}\|_{\infty} - \|C_{11}^{-1}C_{12}s^{(2)}\|_{\infty} \geq 1 - (1 - \eta) = \eta$. $\gamma > 0$ is trivial by noting that $\|C_{11}v\|_{\infty} = 0 \Leftrightarrow v = 0$.
\begin{thm}
  \label{Consistency}
  We can find $\tilde{a}, \tilde{b}$ dependent on $C$, typically, $\tilde{a} = \gamma, \tilde{b} = \|C\|_{\infty}$, such that if $\lim \lambda_n a_n < \tilde{a}$ and $\lim \lambda_n b_n > \tilde{b}$, then LAGS is consistent and has pseudo-hard thresholding property.
\end{thm}
The proof is given in Section 5. Theorem \ref{Consistency} gives another hint that the weights will play an important role for the effectiveness of (\ref{lags}). To clarify the pseudo-hard thresholding property, it is helpful to interpret LAGS geometrically. $l(\beta;\lambda_n) = \frac{1}{n}\|X^{T}(y - X\beta)\|_1 + \lambda_n\sum_{i = 1}^pw_i(y;X;n)|\beta_i|$ is piecewise linear in $(p + 1)-$dimensional space. We assume there are no flat regions---in which $l(\beta)$ are constant, then $l(\beta)$ must obtain its minimum on some breaking point $\beta'$ otherwise there exist descent directions (consider the simplex algorithm). When $n$ is fixed, as we change $\lambda_n$ with a tiny amount, the value in each breaking point may change, but it is possible that no values on other breaking points catch up $l(\beta')$ after the change; if this is the case, $\beta'$ will still be the minimizer even though the $\lambda_n$ is different. To illustrate this point graphically, let us consider a toy 1-dimensional example $y = (7, 2, 4, 2)^T$, $X = (2, 3, 5, 7)^T$
\[l(\beta) = |7 - 2\beta| + |2 - 3\beta| + |4 - 5\beta| + |2 - 7\beta| + \lambda|\beta|\]
When $\lambda = 1, 2, 5$, $l(\beta;1)$ and $l(\beta;2)$ are both minimized at $\beta = 2 / 3$, while $l(\beta; 5)$ at $\beta = 2/7$ as shown in Figure \ref{one_dim_pht}. It implies even though $\lambda$ increases from 1 to 2, the values at other breaking points do not catch up $l(2/3)$, but when increases to 5, $l(2/3)$ is caught up by $l(2/7)$. The consistency can be understood in a similar fashion: as $n$ increases, the number of breaking points increases exponentially, which provides more candidates to solve (\ref{lags}). As a consequence, the probability to discover the true model increases.
\par The discrete selection nature of LAGS can also be understood under this framework: $\hat{\beta}$ only moves from one breaking point to another and the breaking points are scatterred in $p-$dimensional space, which implies $\hat{\beta}$ is selected discretely. Fortunately, this feat of LAGS can be achieved by solving a tractable convex program rather than enumerating all the possible breaking points as $l_0$ regularized regression.
\begin{figure}
  \center
  \caption{Graphical illustration of pseudo-hard thresholding and discrete selection process in one dimension. The blue line is $\lambda = 1$; the green line $\lambda = 2$; the red line $\lambda = 5$. The optimal $\beta$s are indicated by the dotted vertical line.}
  \includegraphics[width = 4in, height = 2in]{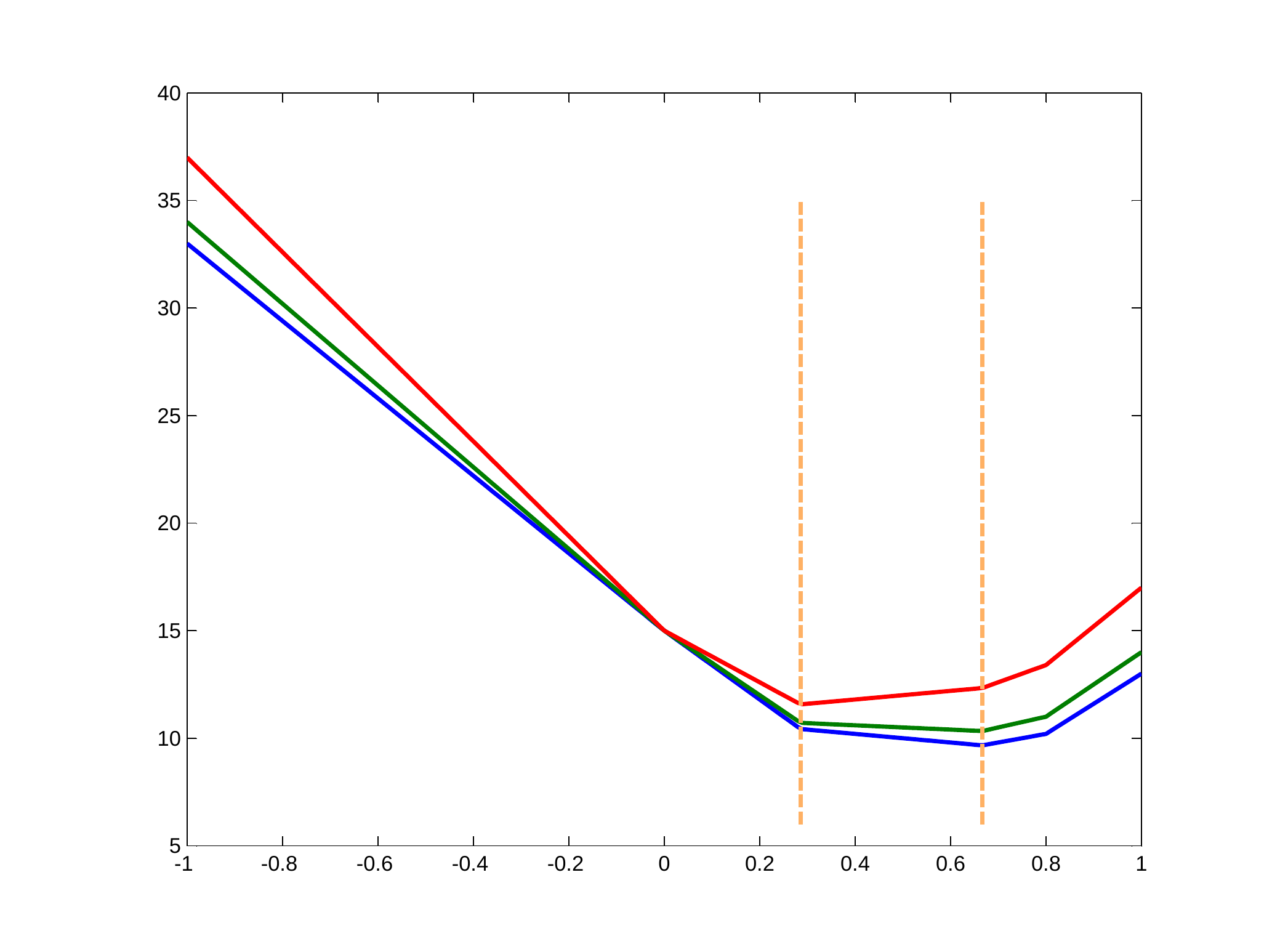}
  \label{one_dim_pht}
\end{figure}
\section{How to choose weights}
\subsection{Dantzig Selector}
The performance of Dantzig Selector (DS) and its similarities with Lasso and LARS are discussed in \cite{Efron}. In their several numerical studies, the coefficient profiles of DS seem to be wilder and have some erratic fluctuations; the prediction accuracy is also inferior. DS can be re-expressed in penalized form as
\begin{equation}
\label{penalizedDS}
\hat{\beta} = \textrm{arg min}\|X^T(y - X^T\beta)\|_{\infty} + \lambda\|\beta\|_1
\end{equation}
We denote the penalized form as $l_{DS}(\beta)$, it is also a piecewise linear function. We argue that one possible reason for these properties in \cite{Efron} is that the penalties on the coefficients are uniform.
\par \emph{Example.} If $l_{DS}(\beta) = \max\{|1 - \beta_1|, |2 - \beta_2|\} + \lambda(|\beta_1| + |\beta_2|)$ by carefully chosen $y$ and $X$. As discussed in Section 3, the piecewise linear function will obtain its minimum at one of its breaking points. In this example, there are total four such points: $(0, 0), (0, 2), (1, 0), (1, 2)$, thus
\[\min l_{DS}(\beta) = \min\{2, 1 + 2\lambda, 2 + \lambda, 3\lambda\}\]
\[\min l_{DS} = \left\{\begin{array}{ll}
2 & \textrm{if } \lambda < 2 / 3, \beta = (0, 0)\\
2 & \textrm{if } \lambda = 2 / 3, \textrm{there are infinit $\beta$s}\\
3\lambda & \textrm{if } \lambda > 2 / 3, \beta = (1, 2)
\end{array}\right.\]
So DS will threshold both $\beta_1$ and $\beta_2$ or keep both intact unless $\lambda = 2 / 3$, where the solution is not unique. LAGS has the similar behavior if the weights are uniform, i.e., if we take $\lambda w_i$s are equal in (\ref{ortholags}). However, Theorem \ref{Consistency} requires the weights converge to different values for true and noisy predictors; hence, we conjecture that if we choose the weights for DS as with LAGS, the counterintuitive behaviors can be eschewed in some extend. We show that numerically in section 8.
\subsection{How to choose $w_i$}
Imposing different weights on the coefficients to enhance predicability is discussed in \cite{Zou2006} and \cite{Wang2007}. Both suggest to use the inverse of ordinary least square estimate as the weight. In the orthonormal case, the adaptive lasso estimates for $\theta$ are obtained by
 \[\hat{\beta}_j^{adaptive} = \textrm{arg}\min_{\beta}\frac{1}{2}(\beta_j^o - \beta)^2 + \lambda\frac{1}{|\beta_j^o|^\gamma}|\beta|\]
 where $j = 1, 2, ..., p$. Therefore, $\hat{\beta}_j^{adaptive} = \textrm{sign}(\beta_j^o)(|\beta_j^o| - \frac{\lambda}{|\beta_j^o|^\gamma})_+$. Compared with Lasso solution (\ref{lassosol}), adaptive lasso shrinks $\beta_j^o$ towards 0 less for larger $\beta_j^o$; as a result, it is shown that it introduces less bias than Lasso. The weight derived in LAD-Lasso \cite{Wang2007} is based on the Bayesian perspective: if each coefficient is double-exponentially distributed with location 0 and scale $\lambda_i$, then the log-likelihood of the posterior is:
\[\sum_{i = 1}^n \log f(\epsilon_i) + \sum_{i = 1}^p\lambda_i|\beta_i| - \log(\lambda_i) + \textrm{constant}\]
minimize it with respect to $\lambda_i$, which leads to $\lambda_i = 1 / |\beta_i|$. However, we do not have the oracle to know $\beta_i$ in advance; hence, at first step, we need a coarse estimate of $\beta_i$, a natural choice will be the $\beta_{OLS}$. Surprisingly, in what follows in this section, it is shown that this is also a good choice for the weights in LAGS.
\begin{thm}\label{thmWeight}
If we choose $w_i = 1 / |\beta_{OLS}|_i$, then asymptotically, conditions for $a_n$ and $b_n$ in Theorem \ref{Consistency} is satisfied, where $\beta_{OLS} = (X^TX)^{-1}X^Ty$.
\end{thm}
What if $p > n$ which is not addressed in \cite{Wang2007,Zou2006}? The number of variables $p$ is larger than the sample size $n$ frequently arises in applications. The ordinary least square fails because the solution of it is not unique. Nevertheless, ridge regression is a shrinkage estimator that can enhance the predictability of the model, so we use ridge solution as the weights for LAGS; moreover, the ridge regression solution can be obtained by simply solving the similar equation, namely $\beta_{ridge} = (X^TX + \phi I)^{-1}y$, where $\phi$ is some positive constant; then $w_i = 1 / |\beta^{ridge}|_i$.
\par There are two nice properties associated with our choice of $w_i$:
\begin{itemize}
    \item[(1)] Without loss of generality, let us consider the scenario that there are several groups of identical predictors, the ridge estimate of $\beta$ will always put equal weights on the identical predictors. The reason is that the optimization problem
        \begin{equation*}
          \min \alpha_1^2 + \alpha_2^2 + ... + \alpha_k^2\quad\quad\quad \textrm{subject to}\quad \alpha_1 + \alpha_2 + ... + \alpha_k = \eta
        \end{equation*}
        will obtain its optimal if and only if $\alpha_1 = \alpha_2 = ... = \alpha_k = \eta / k$. This is desirable since we do not bias towards any predictors.
    \item[(2)] LAGS is a shrinkage estimator in the sense that
    \[\sum_{i = 1}^p \frac{|\hat{\beta}|_i}{|\beta_{OLS}|_i} \leq p\]
    It is easily followed by the fact that $ l(\hat{\beta};\lambda)\leq l(\beta_{OLS};\lambda)$ and $X^T(y - X\beta_{OLS}) = 0$, thus
    \[\frac{1}{n}\|X^T(y - X\hat{\beta})\|_1 + \lambda\sum_{i = 1}^p \frac{|\hat{\beta}|_i}{|\beta_{OLS}|_i}\leq \lambda p\]
    \end{itemize}
    \par Suppose now that $\epsilon \sim \mathcal{N}(0, \sigma^2)$, our next result is non-asymptotic, which claims that under the choice $w_i = 1 / |\beta_{OLS}|_i$, LAGS can accurately identify the underlying model and estimate the coefficients.
    \begin{thm}\label{thm42}
      Under the choice of $w_i = 1 / |\beta_{OLS}|_i$ and suppose that
      \begin{equation}\label{standout_noise}
        \min_{i\in\{1, 2, ..., p_0\}}|\beta_i| > c\sqrt{2\log p}\sigma
      \end{equation}
      \begin{equation}\label{norm_bound}
        \frac{\|C_n\|_{\infty}}{\gamma_n}\leq M
      \end{equation}
      where $\gamma_n = \min_{s\in\Omega}\{\|[C_{11}^n, C_{12}^n](s_1, s_2)^T\|_{\infty}\}$ as in (\ref{defgamma}). If $\xi > 0$ satisfies $(c - \xi) / \xi > M$, then we can choose
      \begin{equation}\label{bound_of_lambda}
        \xi\sqrt{2\log p}\sigma\|C_n\|_{\infty}\leq\lambda\leq (c - \xi)\sqrt{2\log p}\sigma\gamma_n
      \end{equation}
      such that
      \begin{eqnarray}
        \hat{\beta}_i &=& 0, i = p_0 + 1, ..., p\\
        \hat{\beta}_i &=& (\beta_{OLS})_i, i = 1, ..., p_0\\
        \|\hat{\beta} - \beta\|_2^2&\leq& 2\xi^2\cdot p_0 \cdot \log p\cdot\sigma^2\label{estimate}
      \end{eqnarray}
      are satisfied with probability at least $(1 - \pi^{-1/2}\xi^{-1}(n\log p)^{-1/2}\kappa p^{-n\xi^2 / \kappa^2})^p$, where $\kappa$ is a constant dependent on $C_n^{-1/2}$.
    \end{thm}
    In words, the nonzero coefficients should significantly stand above the noise as indicated by (\ref{standout_noise}) and $\|C_n\|_{\infty}/\gamma_n$ is uniformly bounded by $M$ as in (\ref{norm_bound}), $\xi$ is chosen to make the set (\ref{bound_of_lambda}) nonempty. If all these conditions are satisfied, LAGS identifies the correct variables and only these with large probability. Moreover, the coefficients of identified variables are set to be the OLS estimate, which is analogous to hard thresholding. The accuracy of LAGS is quantified by (\ref{estimate}), the mean squared error is proportional to the true number of variables times the variance of the noise with the logarithmic factor---$\log p$, which is unavoidable since we do not have the oracle to know the set of true predictors in advance \cite{Candes2009,Candes2007}.
\section{Proofs of Theorems}
\subsection{Proof of Theorem \ref{Consistency}}
\begin{proof}
  LAGS is a convex program, thus it obtains its minimum at some $\hat{\beta}$. By the theory of convex optimization, the subgradient at $\hat{\beta}$ is zero
  \begin{equation}
    \label{subgradient}
    \begin{split}
    \nabla l(\hat{\beta}) =& -\frac{X^TX}{n}\textrm{sign}(X^Ty - X^TX\hat{\beta}) \\
    &+ \lambda_n(w_1\textrm{sign}(\hat{\beta}_1), ..., w_{p_0}\textrm{sign}(\hat{\beta}_{p_0}), w_{p_0 + 1}\textrm{sign}(\hat{\beta}_{p_0 + 1}), ..., w_p\textrm{sign}(\hat{\beta}_p))^T\\
    =& 0
    \end{split}
  \end{equation}
  where $\textrm{sign}(x) = \frac{x}{|x|}$ for $x\neq 0$ and $\textrm{sign}(x) \in [-1, 1]$ for $x = 0$. We take $\tilde{b} = \|C\|_{\infty}$. Since $b = \lim \lambda_n b_n > \|C\|_{\infty}$, for any $\delta\in(0, (b - \tilde{b}))$, there exits $N_1(\delta)$ and $n > N_1(\delta)$ such that $\|C\|_{\infty} + \delta < \lambda_n b_n$. Similarly, since $C_n\rightarrow C$ as in condition (b), there exits $N_2(\delta)$ and $n > N_2(\delta)$ such that $\|C_n\|_{\infty}\leq \|C\|_{\infty} + \delta$.
  \par The optimal condition (\ref{subgradient}) implies
  \begin{equation}
    \label{thm31redun}
    \Big[\frac{X^TX}{n}\textrm{sign}(X^Ty - X^TX\hat{\beta})\Big]_i = \lambda_n w_i\textrm{sign}(\hat{\beta}_i)
  \end{equation}
  Notice that when $p_0 + 1 \leq i\leq p$, \[\lambda_n b_n \leq\lambda_n w_i\] and \[\Big[\frac{X^TX}{n}\textrm{sign}(X^Ty - X^TX\hat{\beta})\Big]_i\leq\|\frac{X^TX}{n}\|_{\infty}\]
  Hence \begin{equation}\label{PHTcond1}\lambda_n b_n |\textrm{sign}(\hat{\beta}_i)| \leq\|\frac{X^TX}{n}\|_{\infty}\end{equation}
  for $n > N(\delta) = \max\{N_1(\delta), N_2(\delta)\}$, we have $|\textrm{sign}(\hat{\beta}_i)| < 1$, which implies $\hat{\beta}_i = 0$.
  \par Analogously, we take $\tilde{a} = \gamma$, where $\gamma$ is defined in Section 3; when $1\leq i\leq p_0$ and since $a = \lim\lambda_n a_n < \tilde{a}$, for any $\delta'\in(0, (\tilde{a} - a))$, there exists $N(\delta')$ and $n > N(\delta')$ such that $\min_{s\in\Omega}\{\|[C_{11}^n, C_{12}^n]s\|_{\infty}\} > \gamma - \delta'$ and $\lambda_n a_n \leq \gamma - \delta'$. Choose $n > N = \max\{N(\delta), N(\delta')\}$, we have $\hat{\beta}^{(2)} = 0$, then
  \begin{equation}
  \begin{split}
  \textrm{sign}(X^Ty - X^TX\hat{\beta}) &= \textrm{sign}\left(\left(\begin{array}{c}X^{(1)^T}y\\X^{(2)^T}y\end{array}\right)-\left(\begin{array}{c}X^{(1)^T}X^{(1)}\hat{\beta}^{(1)}\\X^{(2)^T}X^{(1)}\hat{\beta}^{(1)}\end{array}\right)\right)\\
  &=\textrm{sign}\left(\left(\begin{array}{c}X^{(1)^T}\epsilon\\X^{(2)^T}\epsilon\end{array}\right)-\left(\begin{array}{c}X^{(1)^T}X^{(1)}(\hat{\beta}^{(1)} - \beta^{(1)})\\X^{(2)^T}X^{(1)}(\hat{\beta}^{(1)} - \beta^{(1)})\end{array}\right)\right)
  \end{split}
  \end{equation}
  We claim that $X^{(1)^T}\epsilon - X^{(1)^T}X^{(1)}(\hat{\beta}^{(1)} - \beta^{(1)}) = 0$.
  \par If it is not true, then \[s' = \textrm{sign}(X^T\epsilon - X^TX(\hat{\beta} - \beta))\in\Omega\] which upon combining with (\ref{thm31redun}) gives
  \begin{equation}\label{PHTcond2}\gamma - \delta' < \|[C_{11}, C_{12}]s'\|_{\infty} < \lambda_n a_n\end{equation}
  contradicts with our choice of $n$.\\
  We impose the superscript $n$ on $\hat{\beta}$ to make it explicitly dependent on $n$. Therefore, for $n > N$,
  \[\frac{1}{n}\|X^{(1)^T}\epsilon - X^{(1)^T}X^{(1)}(\hat{\beta}^{n^{(1)}} - \beta^{(1)})\|_{1} = 0\]
  Combining $\mathbb{E}(\epsilon) = 0$ and the law of large numbers,
  \[\frac{X^{(1)^T}\epsilon}{n}\rightarrow 0\]
  by condition (b)
  \[\frac{X^{(1)^T}X^{(1)}}{n}\rightarrow C_{11}\]
  hence,
  \[\hat{\beta}^{n^{(1)}}\rightarrow \beta^{(1)}\]
  which implies LAGS is consistent.
  \par The pseudo-hard thresholding property holds for $n > N$ as well. In fact, as shown above, if inequalities (\ref{PHTcond1}) and (\ref{PHTcond2}) are satisfied, then we have
  \begin{equation}
    \hat{\beta}^{(1)} = (X^{(1)^T}X^{(1)})^{-1}X^{(1)^T}y
  \end{equation}
  and
  \begin{equation}
    \hat{\beta}^{(2)} = 0
  \end{equation}
  Therefore, for any sequence $\{\lambda_n w_1, \lambda_n w_2, ..., \lambda_n w_p\}$, if $\lambda_n a_n\leq\tilde{a} - \delta'$ and $\lambda_n b_n\geq\tilde{b} + \delta$, then the LAGS solutions are the same.
\end{proof}
\subsection{Proof of Theorem \ref{thmWeight}}
\begin{proof}
The normal equation is
\begin{equation}
\label{normal}
X^TX\beta = X^Ty
\end{equation}
where $y = X^{(1)}\beta^{(1)} + \epsilon$. By condition (a), (b),
\[\frac{1}{n}X^TX \rightarrow C, \frac{1}{n}X^T\epsilon\rightarrow 0, \frac{1}{n}X^TX^{(1)}\beta^{(1)}\rightarrow\left(\begin{array}{c}
C_{11}\\
C_{21}
\end{array}\right)\beta^{(1)}\]
so
\[\hat{\beta}^{(1)} \rightarrow \beta^{(1)}, \hat{\beta}^{(2)}\rightarrow \beta^{(2)} = 0\]
Hence, when $n$ is large enough, we can choose $\lambda_n$, such that $\lambda_n/|\beta_{OLS}|_i < \gamma$ for $i = 1, 2, ..., p_0$ and $\lambda_n/|\beta_{OLS}|_i > \|C\|_{\infty}$ for $i = p_0 + 1, ..., p$.
\end{proof}
\subsection{Proof of Theorem \ref{thm42}}
In order to prove this theorem, we need the following lemmas.
\begin{lem}\label{lemma1}
  If $z\sim \mathcal{N}(0, \Sigma)$, $z\in \mathcal{R}^p$ and $\|\Sigma^{1 / 2}\|_{\infty}\leq c^{-1}$, then the probability
  \[\mathbb{P}(\|z\|_{\infty} \leq t)\geq \Big(1 - \frac{2\phi(ct)}{ct}\Big)^p\]
  where $\phi(t) = (2\pi)^{-1/2}\exp(-t^2/2)$.
\end{lem}
\begin{proof}
  Suppose $v\sim\mathcal{N}(0, I)$ and $z = \Sigma^{1/2}v$,
  \begin{eqnarray}
    &&\int_{\|z\|_{\infty}\leq t}\frac{1}{(2\pi)^{p/2}|\Sigma|^{1/2}}\exp(-\frac{1}{2}z^T\Sigma^{-1}z)dz \\
    &=& \int_{\|\Sigma^{1/2}v\|_{\infty}\leq t}\frac{1}{(2\pi)^{p/2}}\exp(-\frac{1}{2}v^Tv)dv\\
    &\geq& \int_{\|v\|_{\infty} \leq ct}\frac{1}{(2\pi)^{p/2}}\exp(-\frac{1}{2}v^Tv)dv\\
    &\geq& \Big(1 - \frac{2\phi(ct)}{ct}\Big)^p
  \end{eqnarray}
  Here we apply the fact that $\int_{t}^{\infty}\phi(t)dt \leq \phi(t)/t$ for $t > 0$.
\end{proof}
\begin{lem}\label{lemma2}
  Suppose $\|C_n^{-1/2}\|_{\infty}\leq \kappa$, then $|\beta_{OLS}|_i \geq (c-\xi)\sqrt{2\log p}\sigma$ for $i = 1, ..., p_0$ and $|\beta_{OLS}|_i \leq \xi\sqrt{2\log p}\sigma$ for $i = p_0 + 1, ..., p$ are satisfied with probability at least $\Big(1 - \frac{2\phi(\xi\sqrt{2n\log p} / \kappa)}{\xi\sqrt{2n\log p} / \kappa}\Big)^p$ = $(1 - \pi^{-1/2}\xi^{-1}(n\log p)^{-1/2}\kappa p^{-n\xi^2 / \kappa^2})^p$.
\end{lem}
\begin{proof}
  Since $y = X^{(1)}\beta^{(1)} + \epsilon$
  \begin{eqnarray}
    \beta_{OLS} &=& (X^TX)^{-1}X^Ty\\
    &=&\left(\begin{array}{c}
      \beta^{(1)}\\
      0
    \end{array}\right) + \left(\begin{array}{cc}
      X^{(1)^T}X^{(1)} & X^{(1)^T}X^{(2)}\\
      X^{(2)^T}X^{(1)} & X^{(2)^T}X^{(2)}
    \end{array}\right)^{-1}
    \left(\begin{array}{c}
    X^{(1)^T}\epsilon\\
    X^{(2)^T}\epsilon
    \end{array}\right)\\
  \end{eqnarray}
  Denote $\zeta = \left(\begin{array}{cc}
      X^{(1)^T}X^{(1)} & X^{(1)^T}X^{(2)}\\
      X^{(2)^T}X^{(1)} & X^{(2)^T}X^{(2)}
    \end{array}\right)^{-1}
    \left(\begin{array}{c}
    X^{(1)^T}\epsilon\\
    X^{(2)^T}\epsilon
    \end{array}\right)$, which is a Gaussian random vector with mean 0 and variance $C_n^{-1}\sigma^2/n$. Hence by applying Lemma \ref{lemma1} and noting $\|C_n^{-1/2}\sigma/\sqrt{n}\|_{\infty} \leq \kappa\sigma / \sqrt{n}$
    \[\mathbb{P}(\|\zeta\|_{\infty}\leq \xi\sqrt{2\log p}\sigma) \geq \Big(1 - \frac{2\phi(\xi\sqrt{2n\log p} / \kappa)}{\xi\sqrt{2n\log p} / \kappa}\Big)^p\]
\end{proof}
\par Theorem \ref{thm42} is a consequence of the two lemmas.
\begin{proof}
  Since $X^T(y - X\beta_{OLS}) = 0$, it implies
  \[l(\beta) = \frac{1}{n}\|X^Ty - X^TX\beta\|_1 + \lambda\sum_{i = 1}^p w_i|\beta_i| = \frac{1}{n}\|X^TX(\beta_{OLS} - \beta)\| + \lambda\sum_{i = 1}^p w_i|\beta_i|\]
  Choosing $\lambda$ satisfying (\ref{bound_of_lambda}) and applying Lemma \ref{lemma1} and Lemma \ref{lemma2}, we take the subgradient of $l(\beta)$ and use the same arguments as in the proof of Theorem \ref{Consistency}. With large probability, we have
  \[\hat{\beta}_i = 0, i = p_0 + 1, ..., p\]
  and
  \[\Big(X^{(1)^T}X^{(1)}, X^{(1)^T}X^{(2)}\Big)(\hat{\beta}^{(1)} - \beta_{OLS}^{(1)}) = 0\]
  which implies
  \[\hat{\beta}_i = (\beta_{OLS})_i, i = 1, ..., p_0\]
  Hence
  \begin{eqnarray}
  \|\hat{\beta} - \beta\|_2^2 &=& \left\|\left(\begin{array}{c}
    \beta_{OLS}^{(1)}\\
    0
  \end{array}\right) - \left(\begin{array}{c}
    \beta^{(1)}\\
    0
  \end{array}\right)\right\|_2^2\\
  &=& \left\|\left(\begin{array}{c}
    \zeta^{(1)}\\
    0
  \end{array}\right)\right\|_2^2\\
  &\leq& 2\xi^2\cdot p_0 \cdot \log p\cdot\sigma^2
  \end{eqnarray}
  by $\|\zeta\|_{\infty}\leq \xi\sqrt{2\log p}\sigma$.
\end{proof}
\section{Computation and Implementation}
LAGS is a convex program, thus global minimum is assured. There are a lot of algorithms available to solve (\ref{lags}), such as the subgradient method, interior-point methods. However, like Dantzig Selector \cite{Candes2007}, LAGS can also be reformulated as a linear program.
\par Denote $|(X^T(y - X\beta))_i| = u_i,\ i = 1, ..., p$; $|\beta_i| = v_i, \ i = 1, ..., p$. Then solving LAGS is equivalent to solve the following linear program with inequality constraints
\begin{equation}
  \min_{u, v, \beta} \sum_{i = 1}^pu_i + \lambda\sum_{i = 1}^p w_i v_i
\end{equation}
subject to
\begin{eqnarray}
  &-u \leq X^T(y - X\beta)\leq u\\
  &-v \leq \beta\leq v
\end{eqnarray}
This linear program has $3p$ unknowns and $4p$ constraints. When $p$ is relatively small, e.g., less than 100, solving linear program is more efficient; when $p$ gets large, solving the convex program directly is recommended.
\par If the computing environment contains the routine of solving regression under least absolute deviance criterion (LAD). LAGS can also be passed into the routine by treating an augmented samples and responses. Let us define $(y^*, X^*)$, where $(y_i^*, x_i^*) = ((X^Ty)_i, (X^TX)_i)$ for $1\leq i\leq p$; $(y_{p + i}^*, x_{p + i}^*) = (0, \lambda w_i\mathbf{e}_j)$ for $1\leq i \leq p$, $\mathbf{e}_j$ is the $j$th row of identity matrix. Then it can be verified that
\[\hat{\beta} = \textrm{arg min}\|y^* - X^*\beta\|_1\]
So LAGS can be solved without much programming effort.
\par In most applications, we need to run a sequence of sparsity parameter $\lambda$ and choose the optimal one based on some criteria such as cross-validation. An efficient way to accomplish this is to pass the previous solution as the ``warm start'' for the new value of $\lambda$. The justification of this technique is that in simplex algorithm, each iteration tries to find a better candidate solution at the vertices adjacent to the current one, if the difference between $\lambda$s are small, the new minimum should be close to the previous one, hence the new solution can be found in a few iterations.
\section{Numerical Simulation and Example}
\subsection{Prostate Cancer Data}
For the sake of illustrating the Pseudo-Hard Thresholding property, we study the simple yet popular example---Prostate Cancer data \cite{Stamey1989,Tibshirani1996}. The response---logarithm of prostate-specific antigen (lpsa) is regressed on  log(cancer volume) (lcavol), log(prostate weight) (lweight), age, the logarithm of the amount of benign prostatic hyperplasia (lbph), seminal vesicle invasion(svi), log(capsular penetration) (lcp), Gleason score (gleason) and percentage Gleason score 4 or 5(pgg45).
\par The data set is divided into two parts: a training set of 67 observations and a test set of 37 observations. It clearly shows in Figure \ref{prostate_profile} that the coefficients profiles of LAGS and $l_0$ penalty are quite similar. As $\lambda$ increases, the predictors are excluded from the model with the same order at almost the same $\lambda$s. The discrete selection processes of LAGS and $l_0$ penalty are indicated by the jumps and constant segments in the profiles: the jumps means the predictors are either included in the model or not; while the constant segments means the coefficients of the included predictors are unchanged even though $\lambda$ increases. Another interesting observation of the profiles is that the roles of some predictors are downplayed when there are many predictors included, which are mainly caused by the high correlation among them; however, as the included predictors are fewer, their roles are shown up and even increase---refer to the brown line in Figure \ref{prostate_profile}, coefficient of lcavol. The right panel is the prediction errors of LAGS and $l_0$ penalty on the test set respectively, they are also piecewise constant functions of $\lambda$.
\begin{figure}[!bhp]
  \center
  \includegraphics[width = 5in, height = 3in]{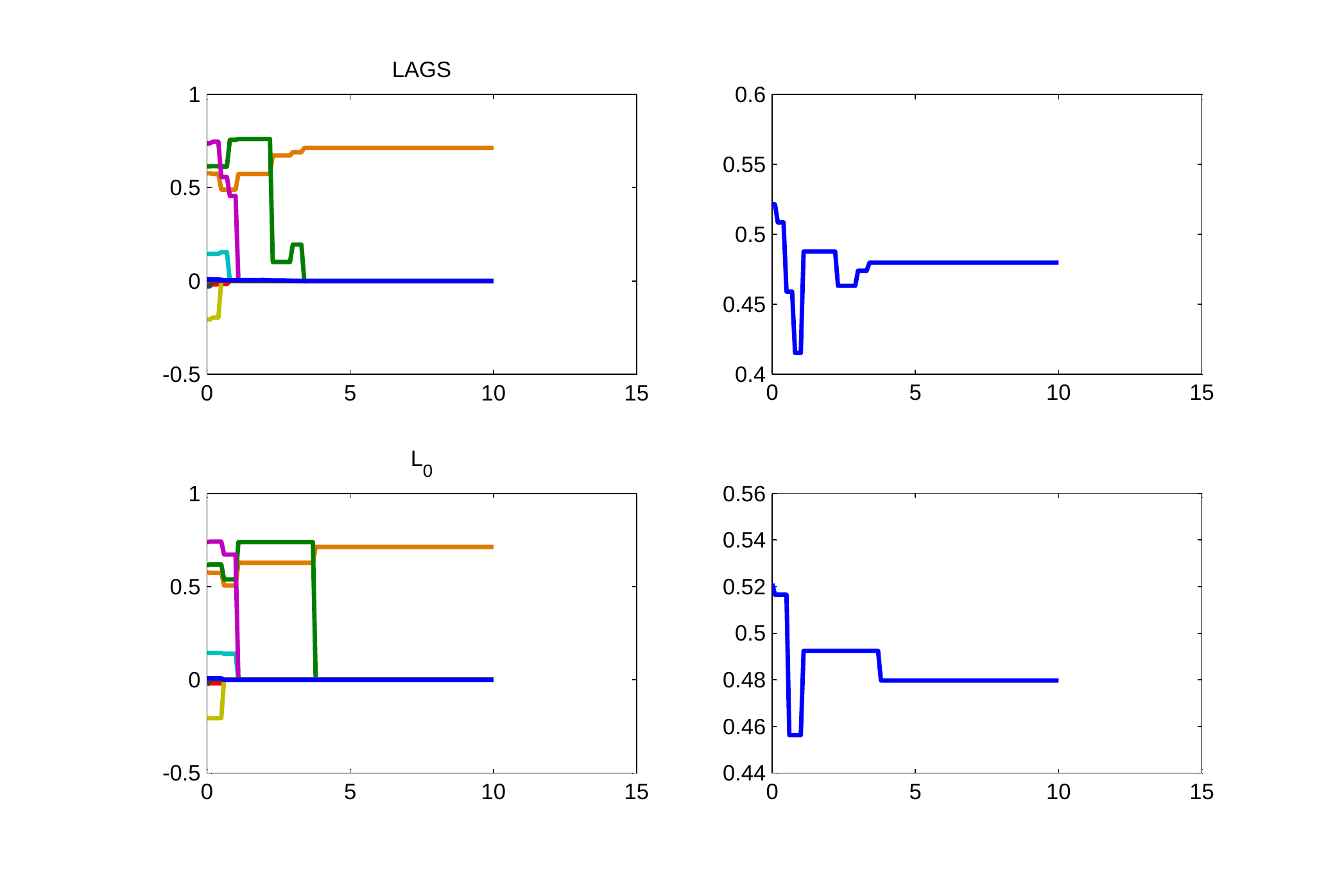}
  \caption{Coefficients profiles as a function of sparsity parameter $\lambda$. There are eight predictors, all the predictors are first centered and standardized before passing to the solvers, then the outputs of solvers are transformed back to the original scale. The right panel is the prediction error for the test data.}
  \label{prostate_profile}
\end{figure}
\par In contrast, the continuous shrinkage property of Lasso are demonstrated in Figure \ref{prostate_lasso}. The \emph{general} trend of the coefficients is decreasing as $\lambda$ increases (this is not true in general, there exist examples that some coefficients increase even though $\lambda$ increases). The prediction error is also larger than that of LAGS and $l_0$ penalty; but care should be taken that this argument can not be generalized; on the contrary, the discrete selection process usually exhibits higher variability hence higher prediction error.
\begin{figure}[!bhp]
  \center
  \includegraphics[width = 5in, height = 1.5in]{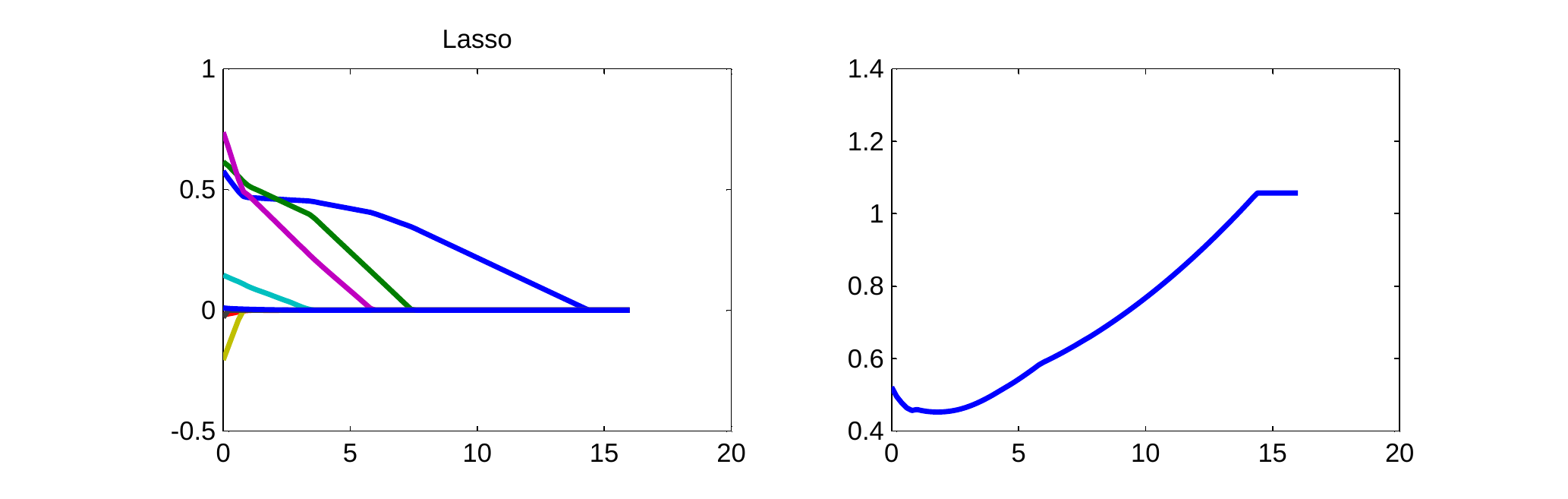}
  \caption{Coefficients profiles as a function of sparsity parameter $\lambda$ by the Lasso.}
  \label{prostate_lasso}
\end{figure}
\subsection{Diabetes Data}
The diabetes data are studied in \cite{Efron,Efron2004}. There are total 442 patients(samples) and 10 predictors in the model. We fit a linear model on this data set. Since LAGS tries to minimize the gradient directly by the $l_1$, it is insightful to compare the gradients between LAGS and Lasso\footnote{Lasso solution is computed by \texttt{R} package \texttt{glmnet} \cite{Friedman2007}}. We have both computed the LAGS and Lasso solution with 5-fold cross validation as in Table \ref{diabetes10}; since LAGS sometimes demonstrates a little higher variability, we pick the most parsimonious model within 45\%-50\% standard error of the minimum instead of the ``one-standard error'' rule. LAGS has 4 nonzero coefficients while Lasso has 5. Moreover, we see in the table that the absolute inner product of predictors with the residue of LAGS is very sparse by its effort to minimizing the gradient directly, whereas that of Lasso is much denser and satisfies
\[X_j^T(y - X\beta^{lasso}) = \lambda \cdot \textrm{sign}(\beta_j^{lasso}),\quad \forall\ \beta_j^{lasso} \neq 0\]
and \[|X_j^T(y - X\beta^{lasso})|\leq \lambda, \quad\forall\ \beta_j^{lasso} = 0\]
It is also notable that the magnitudes of the nonzero coefficients of LAGS are larger than that of Lasso. The reason is that Lasso both shrinks and selects, it tries to relax the penalty on relevant coefficients, as a consequence, some important predictors are downplayed by sharing their weights to others and selected model is relatively dense. This argument can be further verified by comparing the $\|\beta\|_1$: $\|\beta^{lasso}\|_1 \approx 1335$ and $\|\beta^{lags}\|_1 \approx 1617$. Table \ref{corre} summarizes the correlations between the predictors and residue, which shows that LAGS enables some predictors to be exactly orthogonal to the residue. This simple data set shows the superiority of LAGS over Lasso.
\begin{table}[!hbp]
\center
\begin{tabular}{crrrrrr}
  \hline
  \hline
  & \multicolumn{2}{c}{LAGS} & & \multicolumn{2}{c}{Lasso}\\
  \cline{2-3}   \cline{5-6}
  Variable $j$ & $X_j^T(y - X\beta)$ & $\hat{\beta_j}$ && $X_j^T(y - X\beta)$ & $\hat{\beta_j}$\\
  \hline
    1 &  -27.9927    &     0.0000   &&    14.5431 & 0.0000 \\
    2  &       -134.9231 & 0.0000  &&  -111.7310 & -33.3383 \\
    3 &        0.0000 & 604.7797 &&  111.7310 & 508.1903 \\
    4 &        0.0000 & 268.1098 &&  111.7310 & 210.3536 \\
    5  & -53.2906 & -133.8965 && -55.5267 & 0.0000 \\
    6 &         0.0000 &        0.0000 &&  -54.0219 & 0.0000\\
    7 &         119.4116  &0.0000 &&  111.7310 & 138.8478\\
    8 &   73.3477  &      0.0000   &&   66.2507 &     0.0000\\
    9 &          0.0000 &609.8394 &&  111.7310 & 444.5615\\
   10 &   39.7171   &      0.0000  &&  101.9332 & 0.0000\\
   \hline
   Mean Squared Error &\multicolumn{2}{c}{3021} && \multicolumn{2}{c}{3044}\\
   \hline
\end{tabular}
\caption{The gradient $X_j^T(y - X\beta)$ and coefficient $\beta_j$ in each coordinate on the diabetes data with 10 predictors by LAGS and Lasso respectively.}
\label{diabetes10}
\end{table}
\begin{table}
  \center
  \begin{tabular}{|c|rrrrrrrrrr|}
  \hline
  & \multicolumn{9}{c}{correlations between the predictors and residue} &\\
  \hline
  LAGS & -0.02  & -0.12  &  0.00  &  0.00  & -0.05  &  0.00  &  0.10  &  0.06  &  0.00  &  0.03\\
  Lasso & 0.01 & -0.10 &  0.10 &  0.10 & -0.05 & -0.05 &  0.10 & 0.06 &  0.10 &  0.09\\
  \hline
  \end{tabular}
  \caption{The correlations with the residue: $X_j \cdot \textrm{res} / \|X_j\|_2\|\textrm{res}\|_2$}
  \label{corre}
\end{table}
\subsection{Simulated Data}
We compare the simulation performance of LAGS and Sparsenet\footnote{The data sets and Sparsenet solution are computed by \texttt{R} package \texttt{sparsenet} \cite{Mazumder2011}} \cite{Mazumder2011} and Lasso \cite{Tibshirani1996} with regard to training error, prediction error and number of non-zero coefficients in the model. We assume that the predictors and errors are Gaussian distributed. If $X\sim \mathcal{N}(0, \Sigma)$ and $\epsilon\sim \mathcal{N}(0, \sigma^2)$, then the Signal-to-Noise Ratio (SNR) is defined as
\[\textrm{SNR} = \frac{\sqrt{\beta^T\Sigma\beta}}{\sigma}\]
We take $\Sigma = \Sigma(\rho)\in \mathbf{R}^{p\times p}$ with 1's on the diagonal and $\rho$ on the off-diagonal. We generate two data sets with SNR = 2, $\rho = 0.2$ and SNR = 3, $\rho = 0.4$ respectively, the sample sizes $n$ are both 2000 and $\beta = (30, 29, 28, ..., 1, 0_{970})$, $p = 1000$. In order to evaluate the performances of these three algorithms when $p\gg n$, we split the two data sets into training set with 500 samples and testing set with 1500 samples. Before passing the training set into the three algorithms, we first standardize the predictors, then choose the sparsity parameter $\lambda$ by 10-fold cross-validation; since $p \gg n$, the weights for LAGS are set to be the inverse of ridge estimate with $\phi = 0.2$ (since each predictor is standardized, the diagonals of $X^TX$ are 1s, $\phi$ is chosen quite arbitrarily from $(0, 1)$). The results are summarized in Table \ref{n500}, we can observe that Lasso solution is overly dense, whereas Sparsenet solution is overly sparse; LAGS stays between the two and is closer to the true model. In Table \ref{n1500}, the data sets are split into training set with 1500 samples and testing set with 500 samples, the weights are inverse of OLS estimates, $\lambda$ is again chosen by 10-fold cross-validation. In both simulations, the prediction errors of LAGS are slightly larger than that of Sparsent because of the variability caused by discreteness; however, LAGS tends to discover the true models while Sparsenet tends to discover the sparser ones.
\begin{table}
\center
  \begin{tabular}{|c|l|l|l|}
    \hline
    & LAGS & Sparsenet & Lasso\\
    \hline
    \hline
    \multirow{3}{*}{\parbox{1.5cm}{$\rho = 0.2$\\ SNR = 2}} &  \# nonzeros: 33 &  \# nonzeros: 20 & \# nonzeros: 92\\
    & Training Error:  1915.4 & Training Error: 1874.3 & Training Error: 1624.3\\
    & Testing Error: 2258.3 & Testing Error: 2228.3 & Testing Error: 2505.5\\
    \hline
    \multirow{3}{*}{\parbox{1.5cm}{$\rho = 0.4$\\ SNR = 3}} & \# nonzeros: 30 &  \# nonzeros: 22 & \# nonzeros: 112\\
    & Training Error:  710.3 & Training Error: 714.2 & Training Error: 584.1\\
    & Testing Error: 785.5 & Testing Error: 772.0 & Testing Error: 938.8\\
    \hline
  \end{tabular}
  \caption{$n$ = 500, $p$ = 1000, \# test set = 1500. The number of true predictors is 30.}
  \label{n500}
\end{table}
\begin{table}
\center
  \begin{tabular}{|c|l|l|l|}
    \hline
    & LAGS & Sparsenet & Lasso\\
    \hline
    \hline
    \multirow{3}{*}{\parbox{1.5cm}{$\rho = 0.2$\\ SNR = 2}} &  \# nonzeros: 29 &  \# nonzeros: 26 & \# nonzeros: 97\\
    & Training Error:  1900.7 & Training Error: 1854.1 & Training Error: 1811.9\\
    & Testing Error: 2164.9 & Testing Error: 2064.4 & Testing Error: 2274.6\\
    \hline
    \multirow{3}{*}{\parbox{1.5cm}{$\rho = 0.4$\\ SNR = 3}} & \# nonzeros: 29 &  \# nonzeros: 26 & \# nonzeros: 122\\
    & Training Error:  654.3 & Training Error: 594.1 & Training Error: 607.5\\
    & Testing Error: 683.2 & Testing Error: 682.4 & Testing Error: 760.7\\
    \hline
  \end{tabular}
  \caption{$n$ = 1500, $p$ = 1000, \# test set = 500. The number of true predictors is 30.}
  \label{n1500}
\end{table}
\section{Discussion}
\subsection{Adapted version of Dantzig Selector}
In section 4, we argue that one possible explanation of DS in \cite{Efron} is the uniformity of the weight put on each estimator. We adopt the choice of weights as with LAGS. Thus, we need to solve the weighted version of DS
\[\hat{\beta} = \textrm{arg }\min\|X^T(y - X\beta)\|_{\infty} \quad\quad\textrm{subject to}\quad \sum_{i = 1}^p\frac{|\beta_i|}{|\beta_{OLS}|_i} \leq t\]
\begin{figure}
\center
  \includegraphics[width = 5in, height = 2in]{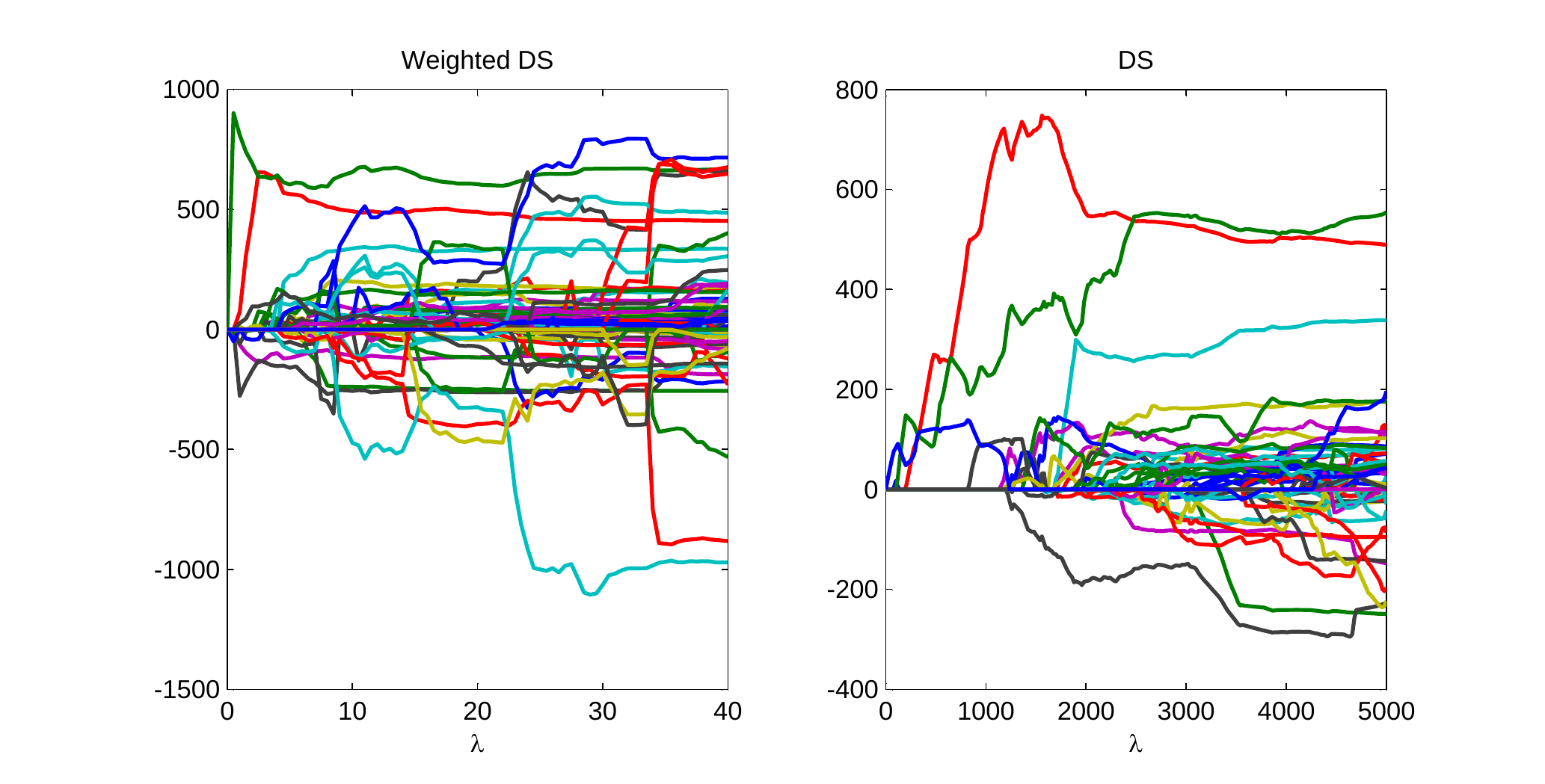}
  \caption{The coefficient profiles of the weighted-DS and DS on diabetes data with interaction terms.}
  \label{weightedDS}
\end{figure} \\
We test our idea on the extended diabetes data where the interactions are included. The left panel of Figure \ref{weightedDS} is the weighted version, while the right panel is the original version which is also appeared in \cite{Efron}. The erratic behavior of DS is relatively mitigated with the weighted DS.
\subsection{Contributions in this paper}
\par In this paper, we propose a new algorithm for linear regression---LAGS and introduce the ``pseudo-hard thresholding'' properties which mimics the $l_0$ regularization. Under mild conditions, we have proved that asymptotically, LAGS is consistent for model selection and parameter estimation. Since the strength of the variable selection algorithms lies in its finite sample performance, we have also established the nonasymptotic theorem which shows that with large probability, LAGS can discover the true model and the error of the estimated parameters is controlled under the noise level. In the proofs of these theorems, we emphasize that the weights on the parameters play a critical role for the effectiveness of LAGS.
\par LAGS is a re-weighted regularization method which is first discussed in adaptive Lasso \cite{Zou2006}, it can be also interpreted as a multistage procedure: first, we provide a very coarse estimate of the parameters of interest; second, based on this estimate, we are able to seek much better ones. Letting the regularization part depend on the data set makes the theories much more difficult; however, it usually results in better performance.
\par The subject of variable selection in linear models has large bodies of literature. The efforts are mainly divided into two streams: on the one hand, the discrete selection procedures of $l_0$ penalty methods such as AIC, BIC are shown to enjoy many nice properties, but they are highly impractical; on the other hand, the continuous shrinkage algorithms such as Lasso and LARS are computationally favorable but they do not have hard-thresholding property and the bias introduced by them is significant sometimes. Our work bridges the gap by pioneering the discrete selection process in the convex world. The attractive properties of LAGS indicate that in some applications, LAGS can be served as a surrogate for $l_0$ penalty and an improved version of the continuous shrinkage methods.
\par There are a group of algorithms on the shelf to solve LAGS. However, since LAGS needs to compute a sequence of solutions like Lasso and Sparsenet, one possible future work is to develop efficient path algorithms such as glmnet and sparsenet.
\bibliographystyle{plain}
  \bibliography{PHT}
\end{document}